\documentclass[letterpaper, 10pt, conference]{ieeeconf} 
\IEEEoverridecommandlockouts
\usepackage{graphicx}
\usepackage{xcolor}
\usepackage{amsmath}
\usepackage{mathtools}
 \mathtoolsset{showonlyrefs}
\usepackage{tikz}
\usepackage{verbatim}
\usepackage{subcaption}
\usepackage{pgfplots}
\pgfplotsset{compat=1.10}
\usepackage{epstopdf} 
\usepackage{relsize}
\usepackage{booktabs}
\definecolor{Gray}{gray}{0.90}
\usepackage{colortbl}

\DeclareMathOperator*{\argmin}{arg\,min}

\usepackage{algorithm}
\usepackage{algpseudocode}
\usepackage{color}
\usepackage{blkarray}
\usepackage{sidecap}
\usepackage{multirow}

\usepackage{amssymb}
\usepackage{url}

\newcommand{\calC}[0]{\mathcal{C}}

\newcommand{\calU}[0]{\mathcal{U}}

\definecolor{amber}{rgb}{1.0, 0.75, 0.0}

\usepackage{hyperref}
\usepackage{resizegather}

\usepackage{makecell}
\usepackage{censor}

\allowdisplaybreaks
\usepackage{stfloats}

\usepackage{bm}
\usepackage{amsmath}
\let\labelindent\relax
\usepackage{enumitem}

\newcommand{\R}{\mathbb{R}}
\newcommand{\cF}{\mathcal{F}}
\newcommand{\cW}{\cF^{\rm agg}}
\newcommand{\vv}[1]{\bm{#1}}
\newcommand{\perpv}[1]{\vv{#1}^{\perp}}

\newtheorem{theorem}{Theorem}

\algtext*{EndIf}
\algtext*{EndFor}
\algtext*{EndProcedure}

\begin{document}{

\title{\bf Communication-Free Obstacle Localization from Aggregate Wrench Measurements in Leader--Follower Cooperative Transport}
\author{ Amin Kashiri, Aditya Mohan, and Yasin Yaz{\i}c{\i}o\u{g}lu
\thanks{This work has been submitted to the IEEE for possible publication.
Copyright may be transferred without notice, after which this version may
no longer be accessible.}
\thanks{A. Kashiri is with the Department of Electrical and Computer Engineering at Northeastern University, Boston, MA ({\tt kashiri.a@northeastern.edu}).}%
\thanks{A. Mohan is with the Department of Electrical and Computer Engineering at Northeastern University, Boston, MA ({\tt lnu.un@northeastern.edu}).}%
\thanks{Y. Yaz{\i}c{\i}o\u{g}lu is with the Departments of Mechanical and Industrial Engineering and Electrical and Computer Engineering at Northeastern University, Boston, MA  ({\tt y.yazicioglu@northeastern.edu}).}%
}
}

\maketitle
\overrideIEEEmargins   
\begin{abstract}
We consider obstacle localization for a team of robots cooperatively transporting a rigid payload without explicit inter-robot communication. A leader robot directs the payload's motion, while follower robots assist and react to locally detected obstacles. The leader measures the followers' aggregate wrench, i.e., the combined force and torque they exert on the payload, but cannot directly distinguish their individual reactions. We design a follower control law that allows the leader to recover obstacle locations from these measurements. Each follower resists motion toward nearby obstacles, resulting in a piecewise-linear relationship between the payload’s translational and angular velocity and the aggregate wrench. Changes between adjacent linear regions reveal an obstacle’s bearing and distance and identify the responding follower. We give sufficient conditions for exact recovery at a fixed payload configuration and develop an adaptive probing procedure in which the leader applies translational and rotational inputs to the payload to obtain the required measurements. We demonstrate the performance of the proposed method in simulations.
\end{abstract}

\section{Introduction}\label{sec:intro}

When a team of robots carries a shared payload without communicating, the payload is their only coordination channel. In leader--follower schemes, a leader that knows the destination steers the payload while each follower reinforces the motion measured at its own grip, so the team transports a rigid body with no communication at all \cite{wang2016kinematic,wang2016force}; the principle has been implemented on robots \cite{connector} and extended to flexible payloads \cite{flex}. Where communication or a shared model is available, transport is instead posed as decentralized optimization \cite{lissandrini2020decentralized,muhammed2026real,dmpclf,tian2025multi} or adaptive control over a broadcast common state \cite{culbertson2021decentralized}. In these works, obstacle information is either absent or supplied directly to the robot that plans around it.

A follower, however, may sense an obstacle that the leader cannot, and without communication it can convey this only through how it pushes on the payload. Such interaction has been used for inference: followers infer the motion intended by a leader that knows the obstacles \cite{bech}, a leader recovers obstacles from the input of the single follower in a two-robot team \cite{bujarbaruah2021learning}, and inverse safety filtering infers constraints from the safety-filtered actions of other agents \cite{isf}. Each attributes the observed response to one robot. With several followers, the leader instead senses one aggregate wrench, three numbers regardless of team size, in which all reactions are summed, so no single measurement reveals which follower responded or to what.

In this paper, we show that the leader can separate these reactions by actively probing the payload. We design a follower law under which each follower resists only motion carrying it toward a nearby obstacle. The map from payload twist to aggregate wrench then becomes piecewise linear, and the change between adjacent linear regions encodes one obstacle's bearing and distance together with the follower that detected it. Recovering piecewise-linear maps from queries is well studied for hybrid systems \cite{pwa} and ReLU networks \cite{carlini}; here, the structure is designed into the team so that every region boundary has a physical meaning. Our contributions are: (i) a communication-free follower law that makes individual obstacle reactions separable in the aggregate wrench; (ii) sufficient conditions under which measurements at a fixed payload configuration uniquely determine every detected obstacle and its responding follower; and (iii) an adaptive probing procedure that selects the leader's test motions to obtain these measurements, evaluated in simulation.

\section{Problem Formulation}\label{sec:bg}

\subsection{Setup}\label{sec:team}

We consider $N$ robots transporting a rigid payload in the plane. Each robot maintains a rigid grip on the payload and can apply
both force and torque. Robot $1$ is the \emph{leader} and alone
knows the destination, while robots $2,\dots,N$ are the \emph{followers}. The followers
do not know their world positions or the destination, and they do not observe
other robots' states or inputs. The robots do not explicitly communicate with each other.

We express vectors in a world frame for the analysis,
while each follower evaluates the control laws in its own grip frame, fixed to
the payload. Let $\vv{r}_i$ denote robot $i$'s \emph{grip offset} from the payload's
center of mass, expressed in the world frame, and let $\vv{p}_i \in \R^2$ denote the
grip's world position. Following the balanced-grip assumption in \cite{wang2018ouijabots},
we assume
$\sum_{i=2}^{N}\vv{r}_i=\vv{0}$.
This arrangement ensures that equal forces at the follower grips produce no
net moment about the payload's center of mass.

The payload has mass $M$, moment of inertia $J$ and viscous-friction coefficient $\mu_v$. Its \emph{pose} is $\vv{x}=[\vv{p}_c^{\!\top},\,\theta]^{\!\top}$, where $\vv{p}_c\in\R^2$ is the position of its center of mass and $\theta\in\R$ its heading. Let $\vv{v}_c=\dot{\vv{p}}_c$ and $\omega=\dot\theta$ denote its center-of-mass velocity and angular velocity, respectively. We collect these quantities into the payload \emph{twist},
$\vv{\nu}=[\vv{v}_c^{\!\top},\,\omega]^{\!\top}=\dot{\vv{x}}$. Robot $i$ applies a force $\vv{f}_i$ and a torque $\tau_i$ at its grip.  Each follower measures its grip velocity $\vv{v}_i$ and the angular velocity $\omega$, and knows its own grip offset, the team size $N$ and the payload parameters.

Obstacles are stationary disks of known radius. A follower detects an obstacle's center location with respect to its grip when the distance to the obstacle's center is less than $d_{\rm act}$. We index each detected obstacle--follower pair by $a$. Thus, an obstacle detected by multiple followers corresponds to multiple pairs. For pair $a$, let $i_a$ denote the follower, $d_a$ the distance from the obstacle center to the grip, and $\vv{n}_a$ the unit vector pointing from the obstacle center to the grip. We call $\vv{n}_a$ the \emph{bearing}. A pair is \emph{active} when $0<d_a<d_{\rm act}$.

The leader has the same local sensing and parameter knowledge as a follower. In addition, it knows all grip offsets, measures its own world pose and the payload's linear and angular acceleration. Using its pose and the known grip geometry, the leader can compute the world positions of all grips.

\subsection{System model}\label{sec:dyn}

The rigid payload connects motion measured at a grip to motion of the whole body. The velocity at grip $i$ is the sum of the center-of-mass velocity and the velocity caused by rotation: $\vv{v}_i=\vv{v}_c+\omega\perpv{r_i}$, where
$\perpv{r_i}=[-r_{iy},\,r_{ix}]^{\!\top}$ is the grip offset rotated by $90^\circ$. We write this relation as 
\begin{equation}\label{eq:B}
  \vv{v}_i=B_i\vv{\nu}, \qquad
  B_i=
  \begin{bmatrix}
    1 & 0 & -r_{iy}\\
    0 & 1 & r_{ix}
  \end{bmatrix}.
\end{equation}

The same grip geometry determines how a local force acts on the payload. For planar vectors, we use $\vv{a}\times\vv{b}=a_xb_y-a_yb_x$ to denote the scalar cross product. 
A force $\vv{f}_i$ contributes the moment $\vv{r}_i\times\vv{f}_i$ about the center of mass. Including the torque $\tau_i$ applied directly through the grip, robot $i$ contributes the wrench 
\begin{equation}\label{eq:gripwrench}
  \cF_i=B_i^{\!\top}\vv{f}_i+\begin{bmatrix}0 & 0 & \tau_i\end{bmatrix}^{\!\top}.
\end{equation}
Thus, $B_i$ maps payload motion to grip velocity, while $B_i^{\!\top}$ maps grip force to payload wrench. 

Following \cite{wang2018ouijabots,wang2016force}, the payload dynamics are
\begin{equation}\label{eq:plant}
  M\dot{\vv{v}}_c
  =\textstyle\sum_i\vv{f}_i-\mu_v\vv{v}_c,
  \qquad
  J\dot\omega
  =\textstyle\sum_i
    \big(\tau_i+\vv{r}_i\times\vv{f}_i\big)
    -\tfrac{\mu_vJ}{M}\omega.
\end{equation}
Each follower assists transport by compensating for $1/N$ of the modeled
viscous friction: $\vv{f}_i^{\rm carry}=\mu_v\vv{v}_c/N$ and
$\tau_i^{\rm carry}=\mu_vJ\omega/(NM)$. The follower obtains $\vv{v}_c$ from its local measurements using $\vv{v}_c=\vv{v}_i-\omega\perpv{r_i}$.

To direct the team toward its destination, the leader selects a desired
translation velocity $\vv{v}_d$ and angular velocity $\omega_d$. For a desired twist
$\vv{\nu}_d=[\vv{v}_d^{\!\top},\,\omega_d]^{\!\top}$, the leader applies
\begin{equation}\label{eq:leader}
  \vv{f}_1=K_1(\vv{v}_d-\vv{v}_c), \qquad
  \tau_1=K_2(\omega_d-\omega)-\vv{r}_1\times\vv{f}_1,
\end{equation}
with positive gains $K_1$ and $K_2$. The last term of the torque cancels the moment produced by the leader's own force. This controller
is also used to generate probing motions for obstacle localization.

Let $\cW=\sum_{i=2}^{N}\cF_i$ denote the followers' aggregate wrench. Summing \eqref{eq:gripwrench} over all robots and using \eqref{eq:plant}, the leader computes it from the measured accelerations and its own wrench:
\begin{equation}\label{eq:aggregate}
  \cW=\begin{bmatrix}M\dot{\vv{v}}_c+\mu_v\vv{v}_c\\ J\dot\omega+\mu_vJ\omega/M\end{bmatrix}-\cF_1 .
\end{equation}

Under the carrying law alone, the $N-1$ followers apply a total force
$(N-1)\mu_v\vv{v}_c/N$ and a total grip torque
$(N-1)\mu_vJ\omega/(NM)$. Their carrying forces also produce moments
about the center of mass, but those moments sum to zero because
$\sum_{i=2}^{N}\vv{r}_i=\vv{0}$.
The aggregate carrying wrench therefore is given by
\[
  \cW=A_c\vv{\nu}, \qquad
  A_c=\frac{N-1}{N}
      \mathrm{diag}\big(\mu_v,\mu_v,\mu_vJ/M\big).
\]
The matrix $A_c$ contains no obstacle information. The next section adds
an obstacle-dependent reaction and shows how the leader can recover
obstacle locations from the resulting wrench.

 \section{Obstacle Recovery}\label{sec:exact}
 
In this section, we design a follower reaction law and establish conditions under which
the resulting aggregate wrench measurements uniquely determine obstacle
locations and the responding followers. The analysis considers a fixed
payload configuration, with the grip positions, obstacle distances $d_a$,
and bearings $\vv{n}_a$ held fixed while the twist varies.
Section~\ref{sec:plant} then describes how the leader can collect
the required measurements using finite-duration probing motions.

\subsection{Follower reaction law}\label{sec:law}
 
Let $\mathcal A_i$ be the set of active obstacle--follower pairs
associated with follower $i$. In addition to its carrying force, follower $i$ applies
\begin{equation}\label{eq:avoid}
  \vv{f}_i^{\rm avoid}
  =\sum_{a\in\mathcal A_i}
    \beta(d_a)
    \max\!\big(0,-\vv{n}_a^{\!\top}\vv{v}_i\big)\vv{n}_a,
\end{equation}
where $\beta:(0,d_{\rm act})\to(0,\infty)$ is a specified,
strictly decreasing function known to all robots. The follower applies $\vv{f}_i=\vv{f}_i^{\rm carry}+\vv{f}_i^{\rm avoid}$, while its directly applied torque remains $\tau_i^{\rm carry}$.
 
Since $\vv{n}_a$ points from the obstacle toward the follower, pair $a$ reacts only
when $\vv{n}_a^{\!\top}\vv{v}_i<0$, that is, when the follower moves toward the
obstacle. The reaction force points away from the obstacle, grows with the approach speed and
vanishes at rest. An active pair is \emph{resisting} when the current motion makes its
reaction nonzero.

\subsection{Piecewise-linear map from twist to wrench}\label{sec:struct}
 
By \eqref{eq:B} and \eqref{eq:gripwrench}, an active pair $a$ contributes
\begin{equation}\label{eq:g}
  \cF_a(\vv{\nu})
  =\beta_a\max\!\big(0,-\vv{g}_a^{\!\top}\vv{\nu}\big)\vv{g}_a,
  \quad
  \vv{g}_a=B_{i_a}^{\!\top}\vv{n}_a
  =
  \begin{bmatrix}
    \vv{n}_a\\
    \vv{r}_{i_a}\times\vv{n}_a
  \end{bmatrix},
\end{equation}
where $\beta_a=\beta(d_a)$. The third entry of $\vv{g}_a$ is
$h_{i_a}(\vv{n}_a)$, where $h_i(\vv{n})=\vv{r}_i\times\vv{n}$ denotes follower $i$'s signed moment arm in direction $\vv{n}$. Summing over all active pairs gives $\cW(\vv{\nu})=A_c\vv{\nu}+\sum_a\cF_a(\vv{\nu})$.
 
The vector $\vv{g}_a$ decides both when the pair resists, through
$\vv{g}_a^{\!\top}\vv{\nu}=\vv{n}_a^{\!\top}\vv{v}_{i_a}<0$, and the direction of its
wrench. Positive scaling of the twist scales the wrench without changing which pairs
resist.
 
Pair $a$ starts or stops resisting when $\vv{g}_a^{\!\top}\vv{\nu}$ changes sign. The twists satisfying $\vv{g}_a^{\!\top}\vv{\nu}=0$ form a plane through the origin, which we call the pair's \emph{switching plane}. These planes divide twist space into regions called \emph{cells}. Within a cell, every pair remains on the same side of its switching plane, so the set $S$ of resisting pairs is fixed.
 
For a twist in that cell, a resisting pair $a$ contributes $-\beta_a\vv{g}_a\vv{g}_a^{\!\top}\vv{\nu}$, while a nonresisting pair contributes zero. The aggregate wrench is therefore linear within the cell:
\begin{equation}\label{eq:cell}
  \cW(\vv{\nu})=A_S\vv{\nu}, \qquad
  A_S=A_c-\sum_{a\in S}\beta_a\vv{g}_a\vv{g}_a^{\!\top}.
\end{equation}
 
Each cell is therefore a cone from the origin. Fig.~\ref{fig:sphere} draws the cells
where these cones meet a sphere centered at the origin, each switching plane meeting it
in a \emph{great circle}.
 
\begin{figure}[t]
\centering
\includegraphics[width=0.8\columnwidth]{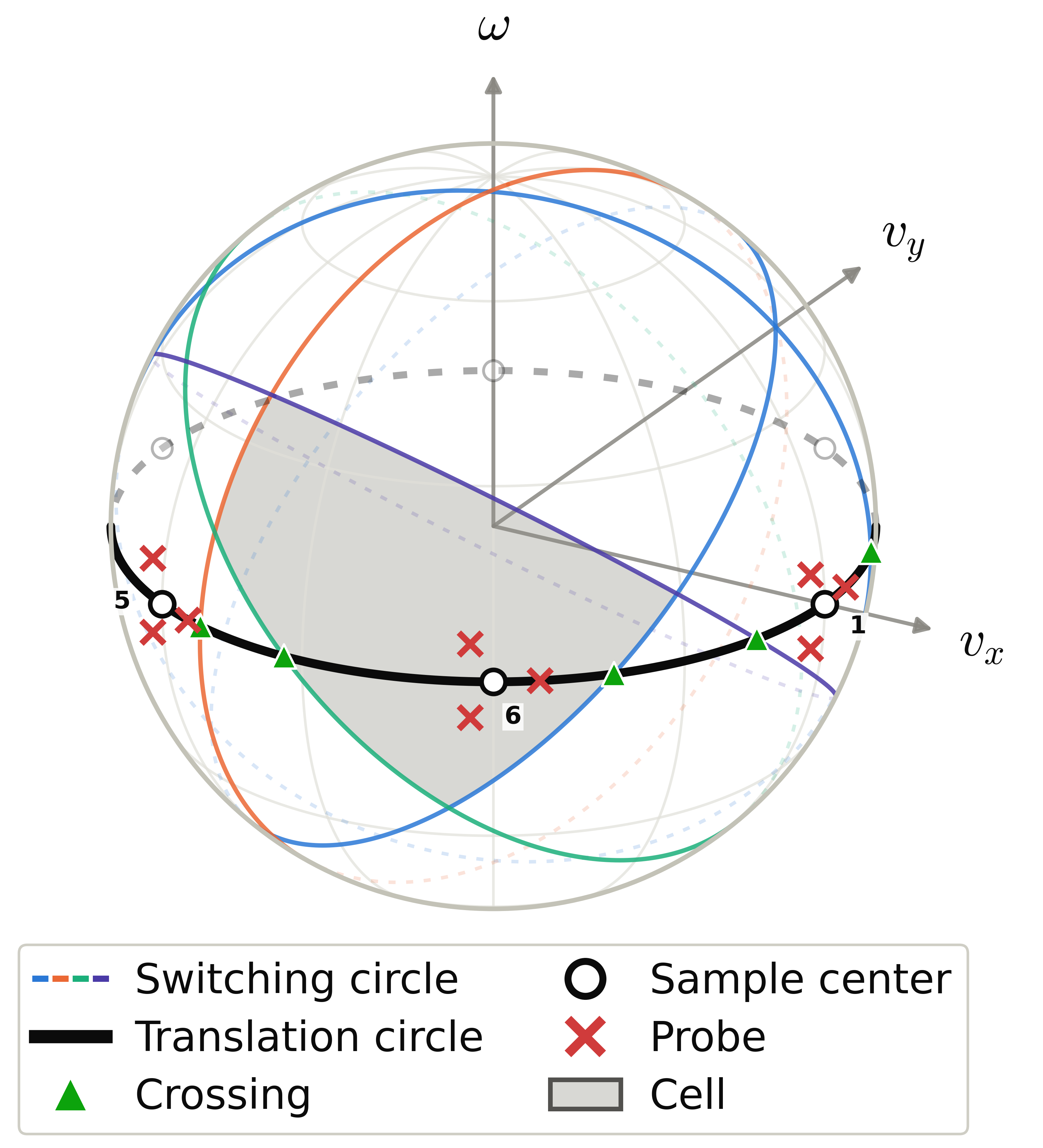}
\caption{The cells of twist space, which are cones, drawn where they meet a sphere
centered at the origin. Each colored great circle is where one pair's switching plane
meets the sphere, and the shaded region is where one cell meets it. The black circle is
the translation circle of Section~\ref{sec:meet}, with crossings (triangles) and
sample centers (circles), numbered as in Fig.~\ref{fig:scene}, each surrounded by its
three probes (crosses, spread exaggerated for visibility). Hidden parts are faded.}
\label{fig:sphere}
\end{figure}

\subsection{Recovery from observations}\label{sec:localmap}
 
An \emph{observation} is a twist $\vv{\nu}$ together with its aggregate wrench
$\cW(\vv{\nu})$. Within a cell the map is linear by \eqref{eq:cell}, so three
observations whose twists $\vv{\nu}_1,\vv{\nu}_2,\vv{\nu}_3$ lie in one cell $S$ and
are linearly independent, i.e., $\det[\vv{\nu}_1\ \vv{\nu}_2\ \vv{\nu}_3]\neq0$,
determine the local matrix $A_S$ of that cell:
\begin{equation}\label{eq:localmap}
  A_S=\big[\cW(\vv{\nu}_1)\ \cW(\vv{\nu}_2)\ \cW(\vv{\nu}_3)\big]
      \big[\vv{\nu}_1\ \vv{\nu}_2\ \vv{\nu}_3\big]^{-1}.
\end{equation}
Two cells are \emph{adjacent} if they lie on opposite sides of one switching plane
and on the same side of every other. The leader does not know in advance which cell
a twist lies in, or which cells are adjacent. For the exact-recovery result in Section~\ref{sec:thm}, we assume
that measurements from two adjacent cells have been identified
and grouped by cell, but the active pair separating
these cells is unknown.

Across the switching plane of pair $a$ the wrench stays continuous, since the reaction
vanishes there, but the local matrix gains or loses the term of $a$. If no other pair
switches there, the matrix after the crossing minus the one before is a rank-one matrix:
\begin{equation}\label{eq:jump}
  \Delta A=A_{\rm after}-A_{\rm before}
  =\mp\beta_a\vv{g}_a\vv{g}_a^{\!\top},
\end{equation}
with the negative sign when pair $a$ starts resisting.
 
Let $\lambda$ be the nonzero eigenvalue of $\Delta A$ and $\vv{q}$ a unit
eigenvector. Then $|\lambda|=\beta_a\|\vv{g}_a\|^2$ and
$\vv{q}=\pm\vv{g}_a/\|\vv{g}_a\|$, and since the first two entries of $\vv{g}_a$ form
the unit bearing, $c=\|\vv{q}_{1:2}\|=1/\|\vv{g}_a\|$. It follows that 
\begin{equation}\label{eq:recover}
  \vv{g}_a=\pm\|\vv{g}_a\|\vv{q}=\pm\vv{q}/c, \qquad
  \beta_a=|\lambda|c^2, \qquad
  d_a=\beta^{-1}(\beta_a).
\end{equation}
 
The sign of $\lambda$ identifies the resisting side:  If $\lambda<0$, pair $a$ resists in the cell after the crossing;
if $\lambda>0$, it resists in the cell before the crossing. We resolve the sign ambiguity in $\vv{g}_a$ by choosing the
candidate satisfying $\vv{g}_a^{\!\top}\vv{\nu}<0$ for an
observed twist in the resisting cell.

The first two entries of $\vv{g}_a$ give the bearing $\vv{n}_a$.
The responding follower $i_a$ is identified by matching
the third entry of $\vv{g}_a$ with the signed moment arms
$h_i(\vv{n}_a)$ of the followers. This match is unique under
the conditions given in Section~\ref{sec:thm}. The obstacle center
is then $\vv{o}_a=\vv{p}_{i_a}-d_a\vv{n}_a$.

 \subsection{Exact recovery guarantee}\label{sec:thm}
 
The reconstruction of Section~\ref{sec:localmap} needs two things: only pair $a$ may switch
between the two cells, and the recovered moment arm must single out its follower. The
following assumptions ensure both.
 
\noindent\textbf{Assumptions.}
\begin{enumerate}[label=\emph{(A\arabic*)},leftmargin=*,itemsep=2pt,beginpenalty=10000]
\item For every active pair $a$, no other follower has the same moment arm in the
bearing direction $\vv{n}_a$: $h_i(\vv{n}_a)\neq h_{i_a}(\vv{n}_a)$ for every
$i\in\{2,\dots,N\}\setminus\{i_a\}$.
\item No two distinct active pairs have parallel or opposite bearings:
$\vv{n}_a\neq\pm\vv{n}_b$ for $a\neq b$.
\end{enumerate}

Under these assumptions, three observations with linearly independent twists in each of two adjacent cells are enough to recover the pair whose switching plane separates them.
 
\begin{theorem}\label{thm:exact}

At a fixed pose, suppose \emph{(A1)} and \emph{(A2)} hold, and let the switching plane of
an active pair $a$ separate two adjacent cells, each containing three observations with
linearly independent twists. Then the reconstruction of Section~\ref{sec:localmap}, applied to the local matrices that
\eqref{eq:localmap} gives from these observations, recovers exactly the follower $i_a$ and
the obstacle center $\vv{o}_a$.
\end{theorem}
 
\begin{proof}
By \eqref{eq:localmap}, the observations in each of these cells determine its local
matrix. Label these matrices $A_1$ and $A_2$. Since the first two entries of
$\vv{g}_a$ are $\vv{n}_a$, \emph{(A2)} makes $\vv{g}_a$ nonparallel to every other
$\vv{g}_b$, so no other pair shares the switching plane of $a$, and only $a$ changes
its resisting status between the two cells. Hence $A_2-A_1$ has the rank-one form in
\eqref{eq:jump}, with a negative nonzero eigenvalue if $a$ resists in cell 2 and a
positive one if it resists in cell 1. The reconstruction of Section~\ref{sec:localmap}
then recovers pair $a$, using an observed twist of the resisting cell to fix the sign
of $\vv{g}_a$. Finally, \emph{(A1)} makes the follower unique.
\end{proof}

\section{The probing procedure}\label{sec:plant}
 
This section presents a probing procedure that searches for
the observations required by Theorem~\ref{thm:exact} and uses
them to estimate obstacle locations. 
 
\subsection{Probes, samples and local matrices}\label{sec:probes}
 
A \emph{probe} is a short out-and-back test motion about a fixed \emph{reference pose}
$\vv{x}_{\rm ref}$. Before each probe, the leader \emph{re-centers} the payload, which
returns it close to $\vv{x}_{\rm ref}$ and reduces the motion left by the previous probe.
It then probes with a twist $\vv{\nu}_p$: it sets the desired twist in \eqref{eq:leader} to
$-\vv{\nu}_p$ for a time $T$, then to $\vv{\nu}_p$, and records the twist and the aggregate
wrench as the payload passes back near the reference position, so that every probe sees
nearly the same map (Fig.~\ref{fig:probe}). Both steps are described in the Appendix.

\begin{algorithm}[!t]
\renewcommand{\algorithmicrequire}{\textbf{Input:}}
\renewcommand{\algorithmicensure}{\textbf{Output:}}
\algnewcommand{\Notation}{\item[\textbf{Notation:}]}
\caption{The probing procedure}\label{alg:probe}
\begin{algorithmic}[1]
\Require
\Statex $n_0$: number of initial samples
\Statex $\rho$: spread of a sample
\Statex $\epsilon_r$: residual tolerance
\Statex $\epsilon_D$: jump tolerance
\Statex $\epsilon_m$: merge tolerance
\Ensure
\Statex $\calC$: candidates $(i,\vv o)$, follower $i$ and obstacle center $\vv o$
\Statex $\calU$: unresolved intervals
\Notation
\Statex sample at $\varphi$: sample of spread $\rho$ centered at angle $\varphi$ on the translation circle
\Statex $\widehat A(\varphi)$, $r(\varphi)$, $\bar{\vv u}(\varphi)$: its fitted matrix, residual and mean measured twist
\Statex \textsc{EigDecomp}$(A)$: eigenpairs $(\lambda_l,\vv q_l)$ of symmetric $A$, $\|\vv q_l\|=1$, $|\lambda_1|\ge|\lambda_2|\ge|\lambda_3|$
\Statex $\hat{\vv g}$, $\hat d$: estimates of $\vv g_a$ and $d_a$
\Statex $h_j$, $\vv p_j$: moment arm and grip position of follower $j$
\State $\calC,\calU\gets\emptyset$
\State sample at $\varphi_k=2\pi k/n_0$, $k=0,\dots,n_0-1$; $\varphi_{n_0}=2\pi$ reuses the sample at $\varphi_0$ \label{alg:init}
\For{$k=0,\dots,n_0-1$}
  \State \Call{Search}{$\varphi_k,\varphi_{k+1}$}
\EndFor
\State remove $(i',\vv o')$ from $\calC$ if an earlier $(i,\vv o)\in\calC$ has $i=i'$ and $\|\vv o-\vv o'\|\le\epsilon_m$ \label{alg:merge}
\Statex
\Procedure{Search}{$\varphi_a,\varphi_b$}
  \If{$\max\{r(\varphi_a),r(\varphi_b)\}>\epsilon_r$} \label{alg:resid}
    \State $\calU\gets\calU\cup\{[\varphi_a,\varphi_b]\}$
  \Else
    \State $\widehat D\gets\widehat A(\varphi_b)-\widehat A(\varphi_a)$ \label{alg:jump}
    \State $(\lambda_l,\vv q_l)_{l=1}^{3}\gets$ \Call{EigDecomp}{$\widehat D$}
    \If{$|\lambda_2|\le\epsilon_D<|\lambda_1|$} \label{alg:rank1}
      \State $c\gets\|(\vv q_1)_{1:2}\|$, \ $\hat\beta\gets|\lambda_1|c^2$
      \If{$c{=}0$ \textbf{or} $\hat\beta$ not in the range of $\beta$} \label{alg:check} \Return \EndIf
      \State $\hat d\gets\beta^{-1}(\hat\beta)$
      \State $\varphi^*\gets\varphi_b$ if $\lambda_1<0$, else $\varphi_a$
      \State $\hat{\vv g}{\gets}\pm\vv q_1/c$, with the sign giving $\hat{\vv g}^{\!\top}\bar{\vv u}(\varphi^*){<}0$
      \State $i\gets\argmin_j|h_j(\hat{\vv g}_{1:2})-\hat g_3|$
      \State $\calC\gets\calC\cup\{(i,\,\vv p_i-\hat d\,\hat{\vv g}_{1:2})\}$ \label{alg:cand}
    \ElsIf{$|\lambda_2|>\epsilon_D$}
      \If{$\varphi_b-\varphi_a\ge2\rho$}
        \State sample at $\varphi_m=(\varphi_a+\varphi_b)/2$ \label{alg:split}
        \State \Call{Search}{$\varphi_a,\varphi_m$}, \Call{Search}{$\varphi_m,\varphi_b$}
      \Else
        \State $\calU\gets\calU\cup\{[\varphi_a,\varphi_b]\}$ \label{alg:narrow}
      \EndIf
    \EndIf
  \EndIf
\EndProcedure
\end{algorithmic}
\end{algorithm}

\begin{figure}[t]
\centering
\begin{subfigure}[b]{0.4\columnwidth}
  \centering\includegraphics[width=\linewidth]{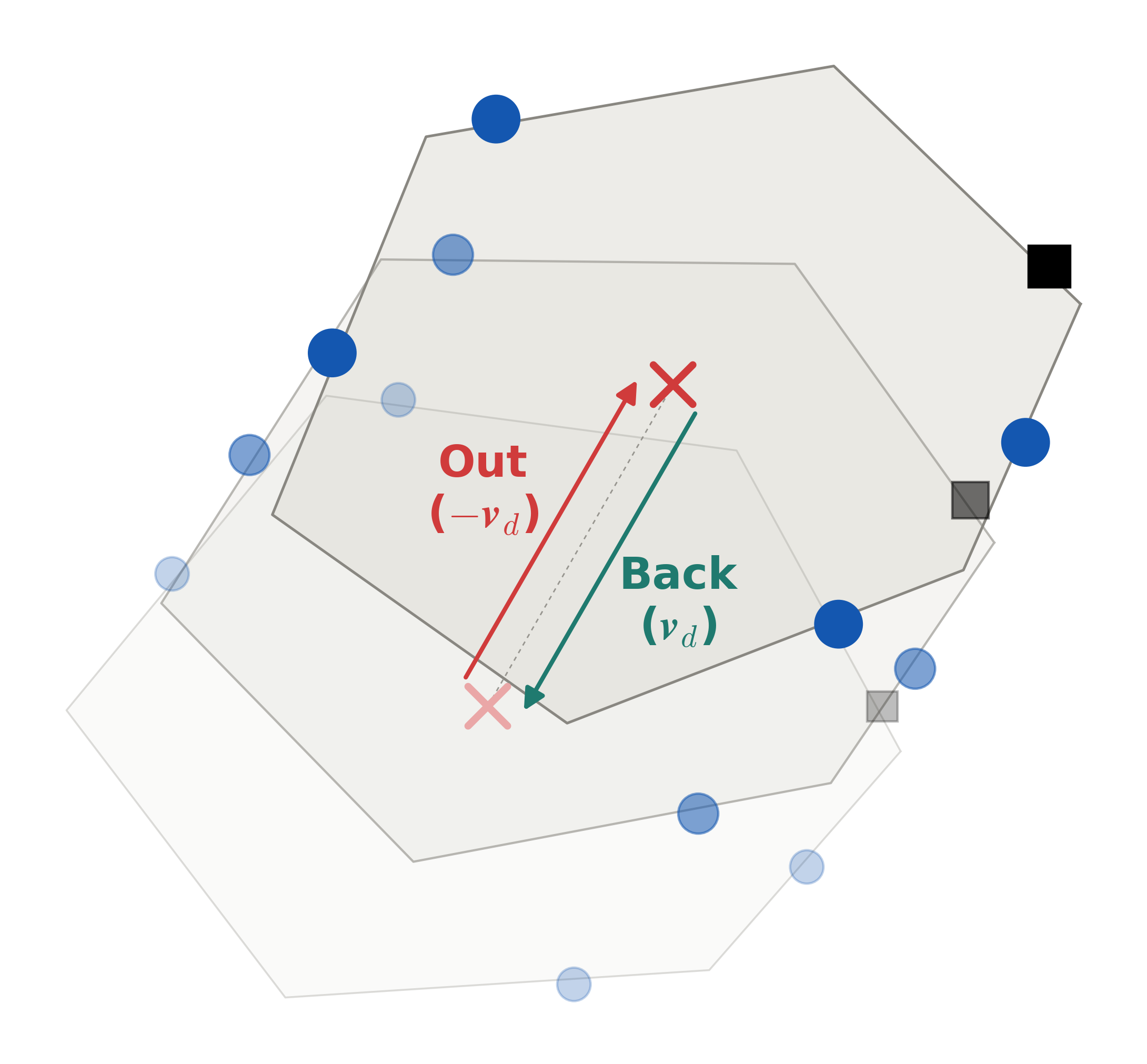}
  \caption{}\label{fig:probe}
\end{subfigure}
\begin{subfigure}[b]{0.55\columnwidth}
  \centering\includegraphics[width=\linewidth]{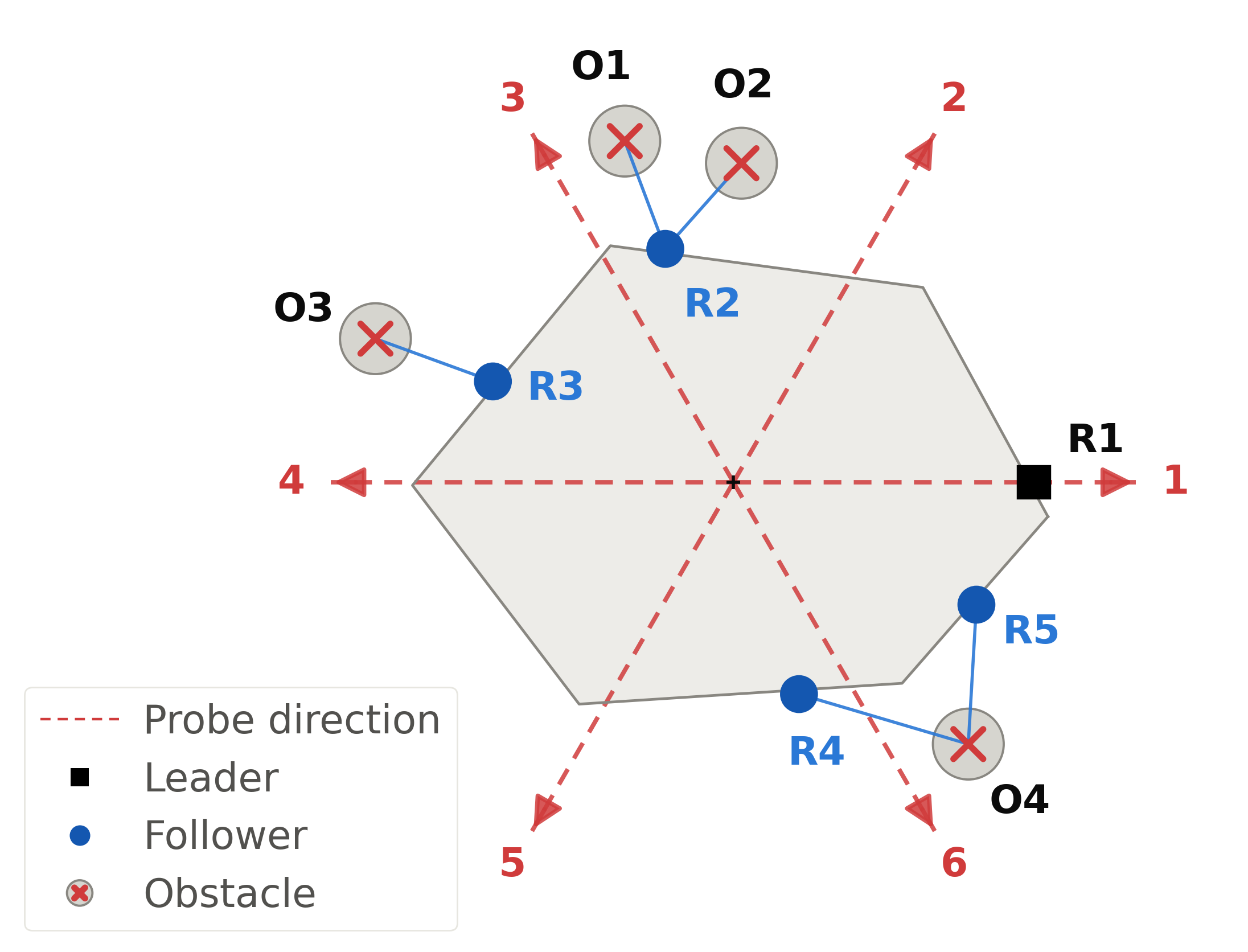}
  \caption{}\label{fig:scene}
\end{subfigure}\hfill
\begin{subfigure}[b]{1.0\columnwidth}
  \centering\includegraphics[width=\linewidth]{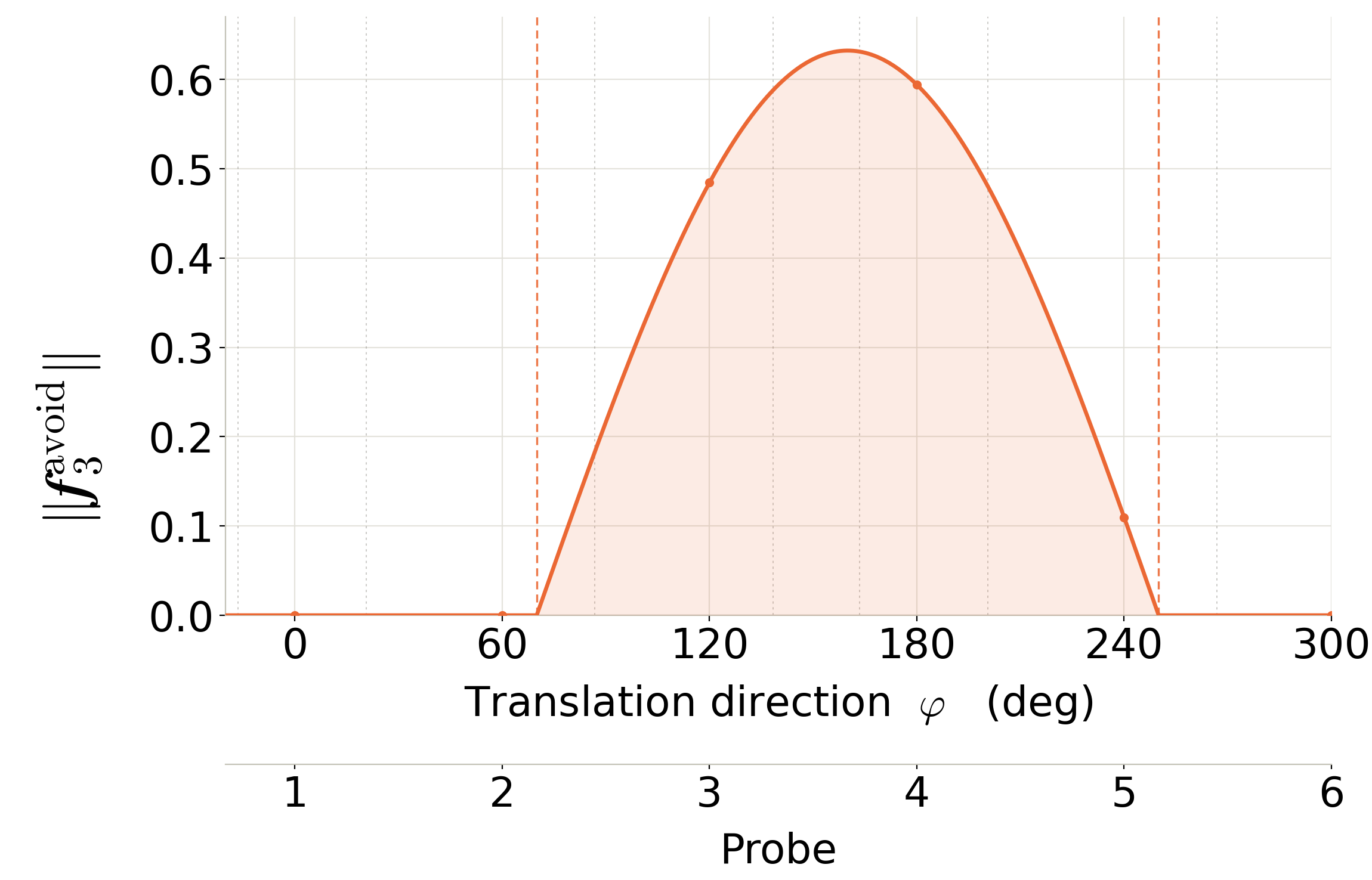}
  \caption{}\label{fig:reaction}
\end{subfigure}
\caption{(a) One probe, with the motion exaggerated: after re-centering, the payload moves out
with $-\vv{\nu}_p$ and back with $\vv{\nu}_p$, and the twist and wrench are recorded as it passes
back near the reference position. In this illustrative probe, the payload moved $2.3$~mm and
turned $0.03^\circ$, and the return ended $3.6~\mu$m from the reference position. (b) Payload,
leader (R1), followers (R2--R5) and the six sample-center directions for $n_0=6$. (c) One
follower's reaction (R3) along the translation circle, zero outside the arc between its
crossings (dashed).}
\label{fig:probing}
\end{figure}
 The leader fits a matrix to each sample. Let $\vv{u}_j$ and $\vv{w}_j$ be the measured
twist and wrench of probe $j$, both multiplied by the same factor so that all twists have
the same magnitude. As shown in Section~\ref{sec:struct}, this rescaling does not
change the local matrix. The fitted matrix is
\[
  \widehat A=\underset{A=A^\top}{\operatorname{arg\,min}}\sum_{j=0}^{2}
     \|A\vv{u}_j-\vv{w}_j\|^2 .
\]
If the three twists lie in one cell, the fit returns its local matrix up to the
measurement error, by \eqref{eq:localmap}. Its \emph{residual} is
$r=\big(\sum_j\|\widehat A\vv{u}_j-\vv{w}_j\|^2\big)^{1/2}$. If the twists lie in different cells, a single symmetric matrix generally cannot fit them accurately, so a large residual
indicates that the sample may straddle a switching plane.

\subsection{The procedure}\label{sec:meet}
 
The procedure has five parameters: the number $n_0$ of initial samples, the spread $\rho$,
and three tolerances $\epsilon_r$, $\epsilon_D$ and $\epsilon_m$
(Algorithm~\ref{alg:probe}).

\emph{Samples.} The leader places $n_0$ equally spaced sample centers on the
\emph{translation circle}, the pure-translation directions $[\cos\varphi,\,\sin\varphi,\,0]^{\!\top}$,
$\varphi\in[0,2\pi)$ (line~\ref{alg:init}), shown in black in Fig.~\ref{fig:sphere} (Fig.~\ref{fig:scene} shows $n_0=6$). Each pair's switching plane meets this circle at two
opposite directions, perpendicular to its bearing, called its \emph{crossings};
Fig.~\ref{fig:reaction} shows one pair's reaction along the circle. A sample whose residual
exceeds $\epsilon_r$ is not used, and the intervals next to it are reported
\emph{unresolved} (line~\ref{alg:resid}).

\emph{Comparing neighbors.} For neighboring samples, the leader computes the \emph{jump}
$\widehat D$, the difference of their fitted matrices (line~\ref{alg:jump}), which by
\eqref{eq:jump} sums the rank-one terms of the pairs that start or stop resisting between them. A nearly zero jump is
treated as no pair switching, although a weak switch, or several switches whose terms cancel, can also produce one. A nearly rank-one jump is
reconstructed as a single pair (line~\ref{alg:rank1}), and any other jump is split at its
midpoint (line~\ref{alg:split}) until the interval is narrower than $2\rho$, after which it
is reported unresolved (line~\ref{alg:narrow}).

\emph{Reconstruction.} Lines~\ref{alg:rank1}--\ref{alg:cand} apply the reconstruction of
Section~\ref{sec:localmap}, with the follower whose moment arm is closest to the recovered
one. If $c=0$, or the recovered strength is not a value of $\beta$, the jump gives no
candidate (line~\ref{alg:check}).

\emph{Merging.} Each pair is found at both of its crossings, so candidates that name the
same follower with obstacle centers within $\epsilon_m$ of each other are merged
(line~\ref{alg:merge}).
 
\section{Results}\label{sec:results}
 
\subsection{Setup}\label{sec:setup}
 
We evaluate the procedure on two sets of scenes. The first contains five designed scenes,
shown in Fig.~\ref{fig:designed}. The second contains thirty random scenes, generated from
fixed seeds, each with 3 to 6 followers at random but balanced grip positions around an ellipse with
semi-axes drawn from $0.8$--$1.6$~m and $0.5$--$1.2$~m, and with the leader at a random
interior point.
Each follower receives an obstacle with probability $0.6$, and then a second with probability
$0.2$, at a random bearing and a distance of $0.55$--$0.9$~m, outside the payload and the robots. Every follower--obstacle pair closer than $d_{\rm act}$ is active,
which gives $147$ active pairs in total. Every run uses the same scenes.
 
The simulations integrate \eqref{eq:plant} with $n_0=6$, the reaction strength $\beta(d)=k\sqrt{1-d/d_{\rm act}}$, and the parameters of
Table~\ref{tab:params}. A pair is \emph{found} when some estimate names its follower and lies within $5$~cm of its
obstacle, and an estimate that matches no pair is a \emph{false positive}.
 
\begin{figure}[t]
\centering\includegraphics[width=\columnwidth]{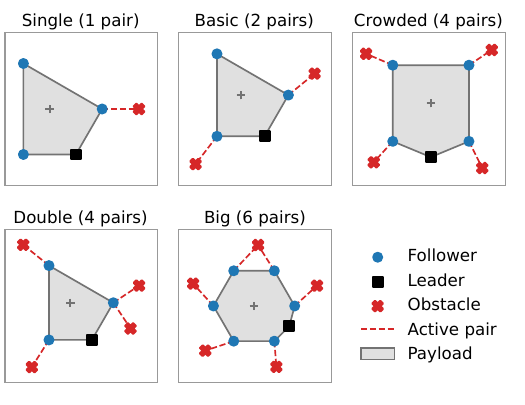}
\caption{The designed scenes at the reference pose, with the number of active pairs.}
\label{fig:designed}
\end{figure}
 
\begin{table}[t]
\caption{Parameters at $K_1=100$~N\,s/m. Some of these parameters scale with $K_1$, with the relation shown in the last column.}\label{tab:params}
\centering\footnotesize
\setlength{\tabcolsep}{4pt}
\begin{tabular}{@{}lllc@{}}
\toprule
 & Quantity & At $K_1=100$~N\,s/m & Scaling \\
\midrule
Payload   & Mass $M$, inertia $J$             & $2$~kg, $0.5$~kg\,m$^2$ & -- \\
          & Carrying gain $\mu_v$             & $2$~N\,s/m              & -- \\
Reaction  & Maximum strength $k$              & $1$~N\,s/m              & -- \\
          & Activation distance $d_{\rm act}$ & $1$~m                   & -- \\
Leader    & Gain $K_2$                        & $25$~N\,m\,s            & $K_1$ \\
Probe     & Speed $\alpha$                    & $3$~mm/s                & -- \\
          & Hold  $T$                            & $0.2$~s                 & $1/K_1$ \\
          & Re-centering duration             & $0.6$~s                 & $1/K_1$ \\
          & Re-centering gain $g_r$           & $15$~s$^{-1}$           & $K_1$ \\
          & Integration step                  & $8$~ms                  & $1/K_1$ \\
Procedure & Initial samples $n_0$             & $6$                     & -- \\
          & Spread $\rho$                     & $0.04$                  & $K_1^{-1/4}$ \\
          & Residual tolerance $\epsilon_r$   & $3\times10^{-6}$        & $1/K_1$ \\
          & Jump tolerance $\epsilon_D$       & $10^{-3}$      & $K_1^{-1/2}$ \\
          & Merge tolerance $\epsilon_m$      & $2$~cm                & -- \\
\bottomrule
\end{tabular}
\end{table}
 
\begin{table}[t]
\caption{Random scenes: pairs found and false positives at each leader gain $K_1$.}\label{tab:res}
\centering\footnotesize
\setlength{\tabcolsep}{4pt}
\begin{tabular}{@{}lcccccc@{}}
\toprule
$K_1$ (N\,s/m) & $100$ & $1000$ & $10^4$ & $10^5$ & $10^6$ & $10^8$ \\
\midrule
Pairs found / $147$ & $54$ & $68$ & $83$ & $101$ & $120$ & $140$ \\
False positives & $3$ & $0$ & $0$ & $1$ & $0$ & $0$ \\
\bottomrule
\end{tabular}
\end{table}
 
\begin{table}[t]
\caption{Designed scenes: pairs found at each leader gain $K_1$, with no false positive.}\label{tab:designed}
\centering\footnotesize
\begin{tabular}{@{}lcc@{}}
\toprule
$K_1$ (N\,s/m) & $100$ & $1000$ \\
\midrule
\textsf{Single} (1) & $1$ & $1$ \\
\textsf{Basic} (2) & $2$ & $2$ \\
\textsf{Crowded} (4) & $4$ & $4$ \\
\textsf{Double} (4) & $4$ & $4$ \\
\textsf{Big} (6) & $4$ & $6$ \\
\midrule
Total ($17$) & $15$ & $17$ \\
\bottomrule
\end{tabular}
\end{table}
 
\subsection{Recovery}\label{sec:limitres}
 
At $K_1=100$~N\,s/m, the procedure finds $15$ of the $17$ designed-scene pairs, missing two
in \textsf{Big}, and $54$ of the $147$ random-scene pairs, with three false positives. The
probes are real motions, so each sample is recorded near, not exactly at, the reference
pose, with twists that differ slightly from the commanded ones.
 
To see how recovery improves as these deviations shrink, we raise $K_1$ up to $10^8$~N\,s/m and change the other parameters with it, following the empirically chosen
schedule of Table~\ref{tab:params}: shorter probes and faster re-centering keep the payload
closer to the reference pose, a smaller spread makes it less likely that a sample straddles
a switching plane, and smaller tolerances let weaker changes in the local matrices be
detected. Recovery improves along this schedule. All $17$ designed-scene pairs are found at
every gain from $1000$~N\,s/m (Table~\ref{tab:designed}), and random-scene recovery rises to
$140$ of the $147$ pairs at $10^8$~N\,s/m, with one false
positive at $10^5$~N\,s/m and none from $10^6$~N\,s/m (Table~\ref{tab:res}).
 
\section{Conclusion}\label{sec:concl}
 
In this work, we considered obstacle localization for a leader--follower team of robots cooperatively transporting a rigid payload without explicit inter-robot communication. Specifically, we investigated how  the leader can localize obstacles outside its sensing range by measuring the aggregate wrench exerted on the payload by the followers. We proposed a reaction law under which each follower resists only motion toward nearby obstacles. We showed that this law makes the map from the payload's twist to the followers' aggregate wrench piecewise linear, with a rank-one change whenever a single follower-obstacle pair starts or stops resisting, and gave conditions under which each obstacle location and its responding follower are recovered in closed form from observations in adjacent cells. We also designed a probing
procedure for the leader to search for such observations by comparing neighboring samples and splitting
the intervals where several pairs switch. Simulation results demonstrated the performance of the proposed method. 

Future work will extend the proposed probing procedure to formally guarantee the collection of observations enabling exact recovery.  We also intend to test our method on real robots.

\bibliographystyle{IEEEtran}
\bibliography{references}

@inproceedings{wang2016kinematic,
  title={Kinematic multi-robot manipulation with no communication using force feedback},
  author={Wang, Zijian and Schwager, Mac},
  booktitle={2016 IEEE international conference on robotics and automation (ICRA)},
  pages={427--432},
  year={2016},
  organization={IEEE}
}

@inproceedings{wang2018ouijabots,
  title={Ouijabots: Omnidirectional robots for cooperative object transport with rotation control using no communication},
  author={Wang, Zijian and Yang, Guang and Su, Xuanshuo and Schwager, Mac},
  booktitle={Distributed Autonomous Robotic Systems: The 13th International Symposium},
  pages={117--131},
  year={2018},
  organization={Springer}
}

@article{wang2016force,
  title={Force-amplifying n-robot transport system (force-ants) for cooperative planar manipulation without communication},
  author={Wang, Zijian and Schwager, Mac},
  journal={The International Journal of Robotics Research},
  volume={35},
  number={13},
  pages={1564--1586},
  year={2016},
  publisher={SAGE Publications Sage UK: London, England}
}

@article{connector,
  author  = {Gong, Zhao and Nie, Zhenguo and Liu, Quan and Liu, Xin-Jun},
  title   = {Design and Control of a Multi-Mobile-Robot Cooperative Transport System
             Based on a Novel Six Degree-of-Freedom Connector},
  journal = {ISA Transactions},
  volume  = {139},
  pages   = {606--620},
  year    = {2023},
  doi     = {10.1016/j.isatra.2023.04.006}
}

@inproceedings{flex,
  author    = {Gombo, Yoshua and Tiwari, Anuj and Devasia, Santosh},
  title     = {Communication-Free Cohesive Flexible-Object Transport Using
               Decentralized Robot Networks},
  booktitle = {American Control Conference (ACC)},
  address   = {New Orleans, LA, USA},
  pages     = {106--111},
  year      = {2021}
}

@inproceedings{lissandrini2020decentralized,
  title={Decentralized nonlinear mpc for robust cooperative manipulation by heterogeneous aerial-ground robots},
  author={Lissandrini, Nicola and Verginis, Christos K and Roque, Pedro and Cenedese, Angelo and Dimarogonas, Dimos V},
  booktitle={2020 IEEE/RSJ International Conference on Intelligent Robots and Systems (IROS)},
  pages={1531--1536},
  year={2020},
  organization={IEEE}
}

@article{muhammed2026real,
  title={Real-time decentralized model predictive control for cooperative multi-robot object transport: experimental validation},
  author={Muhammed, Ibrahim and Nada, Ayman A and El-Hussieny, Haitham},
  journal={Scientific reports},
  volume={16},
  number={1},
  pages={9824},
  year={2026},
  publisher={Nature Publishing Group UK London}
}

@article{dmpclf,
  author  = {Franz{\`e}, Giuseppe and Lucia, Walter and Tedesco, Francesco},
  title   = {A Distributed Model Predictive Control Scheme for
             Leader--Follower Multi-Agent Systems},
  journal = {International Journal of Control},
  volume  = {91},
  number  = {2},
  pages   = {369--382},
  year    = {2018},
  doi     = {10.1080/00207179.2017.1282178}
}

@article{tian2025multi,
  title={Multi-robot cooperative transportation of irregular objects by multi-objective optimization with distributed control},
  author={Tian, Dujie and Zhou, Lelai and Zhang, Chen and Li, Yibin},
  journal={IEEE Robotics and Automation Letters},
  year={2025},
  publisher={IEEE}
}

@article{culbertson2021decentralized,
  title={Decentralized adaptive control for collaborative manipulation of rigid bodies},
  author={Culbertson, Preston and Slotine, Jean-Jacques and Schwager, Mac},
  journal={IEEE Transactions on Robotics},
  volume={37},
  number={6},
  pages={1906--1920},
  year={2021},
  publisher={IEEE}
}

@article{bech,
  author  = {Bechlioulis, Charalampos P. and Kyriakopoulos, Kostas J.},
  title   = {Collaborative Multi-Robot Transportation in Obstacle-Cluttered
             Environments via Implicit Communication},
  journal = {Frontiers in Robotics and AI},
  volume  = {5},
  pages   = {90},
  year    = {2018},
  doi     = {10.3389/frobt.2018.00090}
}

@inproceedings{bujarbaruah2021learning,
  title={Learning environment constraints in collaborative robotics: A decentralized leader-follower approach},
  author={Bujarbaruah, Monimoy and St{\"u}rz, Yvonne R and Holda, Conrad and Johansson, Karl H and Borrelli, Francesco},
  booktitle={2021 IEEE/RSJ International Conference on Intelligent Robots and Systems (IROS)},
  pages={1636--1641},
  year={2021},
  organization={IEEE}
}

@article{isf,
  author  = {Nguyen, Minh and Li, Jingqi and Qu, Gechen and Tomlin, Claire J.},
  title   = {Inverse Safety Filtering: Inferring Constraints from Safety Filters
             for Decentralized Coordination},
  journal = {arXiv preprint arXiv:2604.02687},
  year    = {2026}
}

@article{pwa,
  author  = {Paoletti, Simone and Juloski, Aleksandar Lj. and
             Ferrari-Trecate, Giancarlo and Vidal, Ren{\'e}},
  title   = {Identification of Hybrid Systems: A Tutorial},
  journal = {European Journal of Control},
  volume  = {13},
  number  = {2--3},
  pages   = {242--260},
  year    = {2007},
  doi     = {10.3166/ejc.13.242-260}
}

@inproceedings{carlini,
  author    = {Carlini, Nicholas and Jagielski, Matthew and Mironov, Ilya},
  title     = {Cryptanalytic Extraction of Neural Network Models},
  booktitle = {Advances in Cryptology (CRYPTO)},
  pages     = {189--218},
  year      = {2020}
}

\appendices

\section*{Appendix: implementation details}

Each sample consists of three probes. For a sample centered at angle $\varphi$ with spread
$\rho$, probe $j=0,1,2$ uses the twist
\[
  \vv{\nu}_{p,j}=\alpha\,\frac{\vv{c}+\rho\cos\psi_j\,\vv{t}+\rho\sin\psi_j\,\vv{e}_3}{\sqrt{1+\rho^2}},
  \qquad \psi_j=\frac{2\pi j}{3},
\]
where $\vv{c}=[\cos\varphi,\,\sin\varphi,\,0]^{\!\top}$ is the sample center, the unit vectors
$\vv{t}=[-\sin\varphi,\,\cos\varphi,\,0]^{\!\top}$ and $\vv{e}_3=[0,\,0,\,1]^{\!\top}$ span the
directions around it, and the normalization gives every probe the speed $\alpha$.

Each probe takes three steps. First, the leader re-centers the payload: with
$\vv{x}_{\rm ref}=[\vv{p}_{c,\rm ref}^{\!\top},\,\theta_{\rm ref}]^{\!\top}$, it sets the desired
twist to $-g_r(\vv{x}-\vv{x}_{\rm ref})$ for a fixed duration, with a re-centering gain
$g_r>0$. Second, it sets the desired twist to $-\vv{\nu}_{p,j}$ for a time $T$, then to
$\vv{\nu}_{p,j}$. Third, during this return it records the twist and the aggregate wrench at
the first instant at which $(\vv{p}_c-\vv{p}_{c,\rm ref})^{\!\top}\vv{v}_{p,j}\ge0$, where
$\vv{v}_{p,j}$ is the translational part of $\vv{\nu}_{p,j}$, so that every probe is recorded at
nearly the same center-of-mass position.

\end{document}